\documentclass[10pt,twoside]{article} 

\usepackage[accepted]{rlj}

\usepackage{amssymb}            
\usepackage{mathtools}        
\usepackage{mathrsfs}           
\usepackage{graphicx}          
\usepackage{subcaption}        
\usepackage[space]{grffile}     
\usepackage{url}                
\usepackage{lipsum}            

\usepackage[ruled,vlined,linesnumbered]{algorithm2e}
\usepackage{tikz}
\usetikzlibrary{automata, positioning, arrows.meta, matrix, calc}
\usepackage{float}
\usepackage{amsthm}
\newtheorem{theorem}{Theorem}[section]

\title{A Bellman Optimality Equation for Plasticity}
\setrunningtitle{A Bellman Optimality Equation for Plasticity}
\author{Jeremy Lucas\textsuperscript{1}, Doina Precup\textsuperscript{1}}
\emails{jeremy.lucas@mila.quebec, precupdo@mila.quebec}
\affiliations{
    $^1$Mila -- Quebec Artificial Intelligence Institute \\
    Montreal, Quebec, Canada
}

\usepackage{fancyhdr}
\begin{document}

\maketitle 

\begin{abstract}
In continual reinforcement learning, carefully managing the stability-plasticity tradeoff remains a core challenge. Recent work by \citet{abel2025plasticity} formalized this dilemma by defining plasticity as the generalized directed information from an agent’s observations to its actions, and empowerment as the generalized directed information from its actions to its observations. This formulation successfully reframes the traditional stability-plasticity tradeoff as an empowerment-plasticity tradeoff. However, while extensive literature exists on optimizing for empowerment, there is currently no research addressing the optimization of plasticity under this new definition. This paper presents preliminary work toward optimizing plasticity within Markov decision processes. We show that there exists a Bellman optimality equation for optimizing plasticity similar to previous work for empowerment. 
\end{abstract}

\section{Introduction}
\label{sec:introduction}

Conceptually, empowerment is interpreted as an agent's ability to control its future, whereas plasticity represents its capacity to remain malleable by the environment. Recently, \cite{abel2025plasticity} connected these two domains by formalizing plasticity as the mirror of empowerment. While empowerment describes controllability by maximizing the information flowing from an agent's actions to its subsequent observations, \cite{abel2025plasticity} defined plasticity as the information flowing in the opposite direction. They theorize that designing robust continual learning agents requires maintaining both empowerment and plasticity above critical thresholds, and this claim is the motivation for this work.

To build algorithms that address this claim, both empowerment and plasticity must be actively optimized and controlled. The literature presents numerous methods for maximizing empowerment, ranging from its introduction by \cite{klyubin2005empowerment} to the Bellman formulation of \cite{leibfried2019unified} and the deep variational approaches of \cite{mohamed2015variational}. However, plasticity optimization remains unexplored. Plasticity has primarily been studied through the lens of loss of plasticity in deep architectures \citep{lyle2023understanding, dohare2024loss}, and is often addressed using architectural or training methods \citep{dohare2021continual}. Consequently, no existing work has formalized the direct optimization of the new information-theoretic definition. In this paper, we bridge this gap by deriving a Bellman optimality equation for plasticity. 

The remainder of this paper is structured as follows: Section 2 establishes the necessary mathematical preliminaries; Section 3 details the formal derivation of our Bellman optimality equation; Section 4 introduces our benchmark environment; and Section 5 presents empirical evaluations across both the benchmark and two scaled-up environments.

\section{Background}
\label{sec:background}

\subsection{Empowerment and Plasticity via Directed Information}
From \cite{klyubin2005empowerment} the formal definition of empowerment is the channel capacity between an agent's action channel and the next state that follows these actions. 

\begin{equation}
\mathcal{E}_t(S_t) = \max_{p(A_t^n)} I(A_t^n; S_{t+n} \mid S_t)
\end{equation}

This requires an agent with an open loop system that can produce and execute $n$ actions every state. \cite{capdepuy2010informational} extended this framework by incorporating a feedback loop. They redefined empowerment as the directed information from the agent's actions to its subsequent observations.

\begin{equation}
I(A_t^N \rightarrow O_{t+1}^N) = \sum_{n=0}^{N-1} I(A_t^{n+1}; O_{t+1+n} \mid O_{t+1}^n)
\end{equation}

More recently, \cite{abel2025plasticity} built off of the definition and ideas from \cite{capdepuy2010informational} and defined empowerment through general directed information. They define the empowerment of an agent architecture $\Lambda$ within an environment $e$ relative to an action window $[a:b]$ and observation window $[c:d]$ as:

\begin{equation}
\mathfrak{E}_{c:d}^{a:b}(\Lambda, e) \triangleq \max_{\lambda \in \Lambda} I(A_{a:b} \rightarrow O_{c:d})
\end{equation}

Similar to empowerment, \cite{abel2025plasticity} define the plasticity of an agent $\lambda$ within a set of environments $\mathcal{E}$ relative to an observation window $[a:b]$ and action window $[c:d]$ as:

\begin{equation}
\mathfrak{B}_{c:d}^{a:b}(\lambda, \mathcal{E}) \triangleq \max_{e \in \mathcal{E}} I(O_{a:b} \rightarrow A_{c:d})
\end{equation}

In this paper we are interested in their definition with a single agent $\lambda$ interaction with a specific environment $e$.

\subsection{Value-Based Information Maximization}

While directed information provides a complete formalism to analyze these concepts, it is very difficult to optimize directly in Markov decision processes. Value based methods proposed in \cite{leibfried2019unified, TiomkinTishby2017} formulate generalized Bellman equations to make the optimization for empowerment tractable. \cite{leibfried2019unified} introduced a unified framework for the maximization of reward and empowerment. They introduce two parameters to control focus on reward versus empowerment. They define an empowerment Bellman optimality equation given as the following: 

\begin{equation}
V^*(s) = \max_{\pi_{\text{behave}}, q} \mathbb{E}_{a \sim \pi_{\text{behave}}(\cdot \mid s)} \left[ \alpha R(s, a) + \mathbb{E}_{s' \sim P(\cdot \mid s, a)} \left[ \beta \log \left( \frac{q(a \mid s', s)}{\pi_{\text{behave}}(a \mid s)} \right) + \gamma V^*(s') \right] \right]
\end{equation}

We can now develop a Bellman optimality equation for plasticity similar to this formulation. We focus on the scenario where there are no external rewards (or $\alpha$ is set to $0$) to isolate and examine the information-theoretic value of plasticity and the resulting policies without the influence of external rewards.

\section{The Bellman Optimality Equation for Plasticity}
\label{sec:maximizing_mi}

\subsection{Plasticity and Empowerment Value Channels}
In \cite{leibfried2019unified}, they create a value function that isolates the present and future information. Given $S_t$, they use $I(A_t, S_{t+1} \mid S_t)$ as the step reward and $V(S_{t+1})$ as the future reward which looks like the following.  

\begin{center}
\begin{tikzpicture}[
    >=Stealth, 
    state/.style={minimum size=0.8cm, fill=blue!5, rounded corners=2pt, anchor=center},
    action/.style={minimum size=0.7cm, rounded corners, fill=orange!5, anchor=center},
    dots/.style={minimum size=0.5cm, anchor=center},
    history/.style={fill=yellow!30, draw=yellow!60, thick},
    reward/.style={fill=red!25, draw=red!50, thick},
    future/.style={fill=green!25, draw=green!50, thick}
]
    \draw[red!10, fill=red!5, rounded corners=12pt, dashed, draw=red!40] (1.0, -1.9) rectangle (3.5, 0.5);
    \node[red!80!black, font=\scriptsize\bfseries, anchor=north] at (2.25, -1.9) {$I(A_t; S_{t+1} \mid S_t)$};

    \draw[green!10, fill=green!3, rounded corners=10pt, dotted, draw=green!40] (4.0, -1.9) rectangle (6.5, 0.5);
    \node[green!60!black, font=\tiny\bfseries, anchor=north] at (5.25, -1.9) {$\gamma V(S_{t+1})$};
    
    \draw[green!10, fill=green!3, rounded corners=10pt, dotted, draw=green!40] (6.75, -1.9) rectangle (9.0, 0.5);
    \node[green!60!black, font=\tiny\bfseries, anchor=north] at (7.875, -1.9) {};

    \node[state, history] (S0)    at (0, 0)   {$S_t$};
    \node[state, reward]  (S1)    at (3, 0)   {$S_{t+1}$};
    \node[state, future]  (S2)    at (6, 0)   {$S_{t+2}$};
    \node[dots, future]   (Sdots) at (8.5, 0) {$\dots$};
    \node[state, future]  (Sn)    at (11, 0)  {$S_{t+n}$};

    \node[action, reward] (A0) at (1.5, -1.5) {$A_t$};
    \node[action, future] (A1) at (4.5, -1.5) {$A_{t+1}$};
    \node[action, future] (A2) at (7.25, -1.5) {$A_{t+2}$};
    \node[action, future] (An) at (9.75, -1.5) {$A_{t+n-1}$};

    \draw[->, thick, black] (S0) -- (A0);
    \draw[->, thick, red!70!black] (A0) -- (S1);
    
    \draw[->, thick, black] (S1) -- (A1);
    \draw[->, thick, green!60!black] (A1) -- (S2);
    \draw[->, thick, black] (S2) -- (A2);
    \draw[->, thick, green!60!black] (A2) -- (Sdots);
    \draw[->, thick, black] (Sdots) -- (An);
    \draw[->, thick, green!60!black] (An) -- (Sn);

\end{tikzpicture}
\end{center}

In this paper we formulate the equivalent counterpart for plasticity. We use the step reward $I(S_{t+1}, A_{t+1}| S_t, A_t)$ and the future context $V(S_{t+1}, A_{t+1})$ as the future value.

\begin{center}
\begin{tikzpicture}[
    >=Stealth, 
    state/.style={minimum size=0.8cm, fill=blue!5, rounded corners=2pt, anchor=center},
    action/.style={minimum size=0.7cm, rounded corners, fill=orange!5, anchor=center},
    dots/.style={minimum size=0.5cm, anchor=center},
    history/.style={fill=yellow!30, draw=yellow!60, thick},
    reward/.style={fill=red!25, draw=red!50, thick},
    future/.style={fill=green!25, draw=green!50, thick}
]

    \draw[red!10, fill=red!5, rounded corners=12pt, dashed, draw=red!40] (2.5, -1.9) rectangle (5.0, 0.5);
    \node[red!80!black, font=\scriptsize\bfseries, anchor=north] at (3.75, -1.9) {$I(S_{t+1}; A_{t+1} \mid S_t, A_t)$};

    \draw[green!10, fill=green!3, rounded corners=10pt, dotted, draw=green!40] (5.5, -1.9) rectangle (7.75, 0.5);
    \node[green!60!black, font=\tiny\bfseries, anchor=north] at (6.625, -1.9) {$\gamma V(S_{t+1}, A_{t+1})$};
    
    \draw[green!10, fill=green!3, rounded corners=10pt, dotted, draw=green!40] (8.0, -1.9) rectangle (10.25, 0.5);
    \node[green!60!black, font=\tiny\bfseries, anchor=north] at (9.125, -1.9) {};

    \node[state, history] (S0)    at (0, 0)   {$S_t$};
    \node[state, reward]  (S1)    at (3, 0)   {$S_{t+1}$};
    \node[state, future]  (S2)    at (6, 0)   {$S_{t+2}$};
    \node[dots, future]   (Sdots) at (8.5, 0) {$\dots$};
    \node[state, future]  (Sn)    at (11, 0)  {$S_{t+n}$};

    \node[action, history] (A0) at (1.5, -1.5) {$A_t$};
    \node[action, reward]  (A1) at (4.5, -1.5) {$A_{t+1}$};
    \node[action, future]  (A2) at (7.25, -1.5) {$A_{t+2}$};
    \node[action, future]  (An) at (9.75, -1.5) {$A_{t+n-1}$};

    \draw[->, thick, black] (S0) -- (A0);
    
    \draw[->, thick, black] (A0) -- (S1);
    \draw[->, thick, red!70!black] (S1) -- (A1);
    
    \draw[->, thick, black] (A1) -- (S2);
    \draw[->, thick, green!60!black] (S2) -- (A2);
    \draw[->, thick, black] (A2) -- (Sdots);
    \draw[->, thick, green!60!black] (Sdots) -- (An);
    \draw[->, thick, black] (An) -- (Sn);

\end{tikzpicture}
\end{center}

\subsection{The Bellman Optimality Equation Definition}
We now introduce the Bellman optimality equation for plasticity. We establish its theoretical foundations by proving that the associated Bellman operator is a contraction, with the full proof deferred to Appendix A. 

\begin{equation}
V^*(s_{t-1}, a_{t-1}) = \max_{\pi \in \Pi(s_{t-1}, a_{t-1})} \underset{s_t \sim p}{\mathbb{E}} \left[ \underset{a_t \sim \pi}{\mathbb{E}} \left[ \log \frac{\pi(a_t \mid s_{t-1}, a_{t-1}, s_t)}{q(a_t \mid s_{t-1}, a_{t-1})} + \gamma V^*(s_t, a_t) \right] \right]
\end{equation}

Here, $\Pi(s_{t-1}, a_{t-1})$ denotes the set of policies that match the current policy everywhere except for the non-zero values of $p(\cdot \mid s_{t-1}, a_{t-1})$. The $q$ marginal is defined below: 

\begin{equation}
q(a_t \mid s_{t-1}, a_{t-1}) = \sum_{s_t} p(s_t \mid s_{t-1}, a_{t-1}) \cdot \pi(a_t \mid s_{t-1}, a_{t-1}, s_t)
\end{equation}

To apply this Bellman value function in a tabular learning setting, the policy must be augmented with the historical state and action. For each history context, the action distributions for all next states $s_t$ that can arise from $p(s_t \mid s_{t-1}, a_{t-1})$ must be optimized jointly due to the coupling from the $q$ marginal. A tractable computation for this is described below. 

\subsection{Algorithm Optimization}
The Bellman operator can be restricted to deterministic policies, as proven in the appendix. Therefore, the search set for each history context consists of the max over the distinct combinations of deterministic action distributions. This requires searching through $\vert{}A\vert{}^{\vert{}S\vert{}}$ combinations, where $\vert{}A\vert{}$ is the number of actions and $\vert{}S\vert{}$ is the number of distinct states arising from each history context. This search space becomes rapidly intractable and a further optimization must be made. 

\subsection{Multi-Weight Knapsack Dynamic Programming}
Because the marginal distribution $q(a_t \mid s_{t-1}, a_{t-1})$ is shared, the action choices across all next-states $s_t$ are coupled and cannot be optimized individually. However, this joint search space can be significantly reduced by framing the optimization as a Multi-Weight Knapsack problem. 

In this formulation, each next-state $s_t$ acts as an item for which we must select exactly one action $a_t$. Choosing action $a_t$ contributes a "weight" equal to the transition probability $p(s_t \mid s_{t-1}, a_{t-1})$ to the corresponding coordinate of the $q$ marginal, while yielding a linear future value of $\gamma V(s_t, a_t)$. The dimensions of our knapsack track the accumulated weight distribution across the action space. By using dynamic programming, we track identical accumulated weight states across action-selection paths, allowing us to prune suboptimal combinations and stop redundant computation.

\begin{algorithm}[H]
\SetAlgoLined
\DontPrintSemicolon
\KwIn{State space $\mathcal{S}$, Action space $\mathcal{A}$, transition probabilities $p(s_t \mid s_{t-1}, a_{t-1})$, discount factor $\gamma$, convergence threshold $\theta > 0$}
\KwOut{Optimal value function $V^*(s, a)$ and optimal policy $\pi^*(a_t \mid s_{t-1}, a_{t-1}, s_t)$}

Initialize $V(s, a) \leftarrow 0$ for all $(s, a) \in \mathcal{S} \times \mathcal{A}$\;
Initialize $\pi(a_t \mid s_{t-1}, a_{t-1}, s_t) \leftarrow \mathbf{1}(a_t = a^*)$ where $a^* \sim U(\mathcal{A})$, for all $(s_{t-1}, a_{t-1}, s_t, a_t) \in \mathcal{S} \times \mathcal{A} \times \mathcal{S} \times \mathcal{A}$\;

\Repeat{$\Delta < \theta$}{
    $\Delta \leftarrow 0$\;
    \ForEach{$(s_{t-1}, a_{t-1}) \in \mathcal{S} \times \mathcal{A}$}{
        $v \leftarrow V(s_{t-1}, a_{t-1})$\;
        
        $V(s_{t-1}, a_{t-1}), \pi^*(\cdot \mid s_{t-1}, a_{t-1}, \cdot) \leftarrow \text{Solve Multi-Weight Knapsack DP}(\pi, V, p, \gamma)$\;
        
        $\Delta \leftarrow \max(\Delta, |v - V(s_{t-1}, a_{t-1})|)$\;
    }
}
\caption{Tabular Value Iteration for Plasticity}
\label{alg:value_iteration}
\end{algorithm}

\section{The Plasticity-Empowerment Landscape}
\label{sec:landscape}

To build intuition regarding the dynamics of plasticity and empowerment, we introduce a minimal benchmark environment designed to isolate unique policy interactions between the two metrics. The simplest architecture identified that retains a rich, non-trivial plasticity-empowerment landscape is the Markov Decision Process illustrated in Figure~\ref{fig:mdp_ultra_simple}.

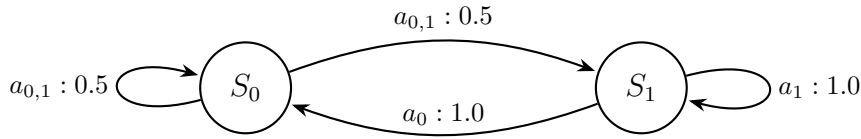
\begin{figure}[h!]
\centering
\begin{tikzpicture}[
    ->, >=Stealth, 
    shorten >=1pt, 
    auto, 
    node distance=4cm, 
    thick,
    state/.style={circle, draw, minimum size=1.2cm, font=\large}
]

    \node[state] (S0) {$S_0$};
    \node[state, right=of S0] (S1) {$S_1$};

    \path (S0) edge[loop left, distance=1.5cm] node[align=center] {$a_{0,1} : 0.5$} (S0);
    \path (S0) edge[bend left=20] node {$a_{0,1} : 0.5$} (S1);

    \path (S1) edge[bend left=20] node[above] {$a_0 : 1.0$} (S0);
    \path (S1) edge[loop right, distance=1.5cm] node {$a_1 : 1.0$} (S1);

\end{tikzpicture}
\caption{The Control-Gated MDP Benchmark.}
\label{fig:mdp_ultra_simple}
\end{figure}

We designate this environment as the \textit{Control-Gated MDP}. Despite its minimal size, it exhibits several compelling properties for analysis:
\begin{enumerate}
    \item \textbf{Computational Efficiency:} Consisting of only two states and two actions, it allows for exact, rapid numerical optimization and complete policy enumeration.
    \item \textbf{Asymmetric Controllability:} It features a highly empowering state ($S_1$) alongside a non-empowering state ($S_0$).
\end{enumerate}

By grid-searching across all valid stochastic policies $\pi(a \mid s)$ at a granularity of $0.01$, we map the discretized, empirical plasticity-empowerment rate landscape shown in Figure~\ref{fig:my_image}. The plasticity and empowerment is a measure of total directed information from states to actions and vice-versa for the given policy and environment. Although both metrics are dependent on the environment, the maximum information rate (in bits) for plasticity is bounded by $\log_2|A|$ while for empowerment it is bounded by $\log_2|S|$.
\begin{figure}[htbp]
    \centering
    \includegraphics[width=0.6\textwidth]{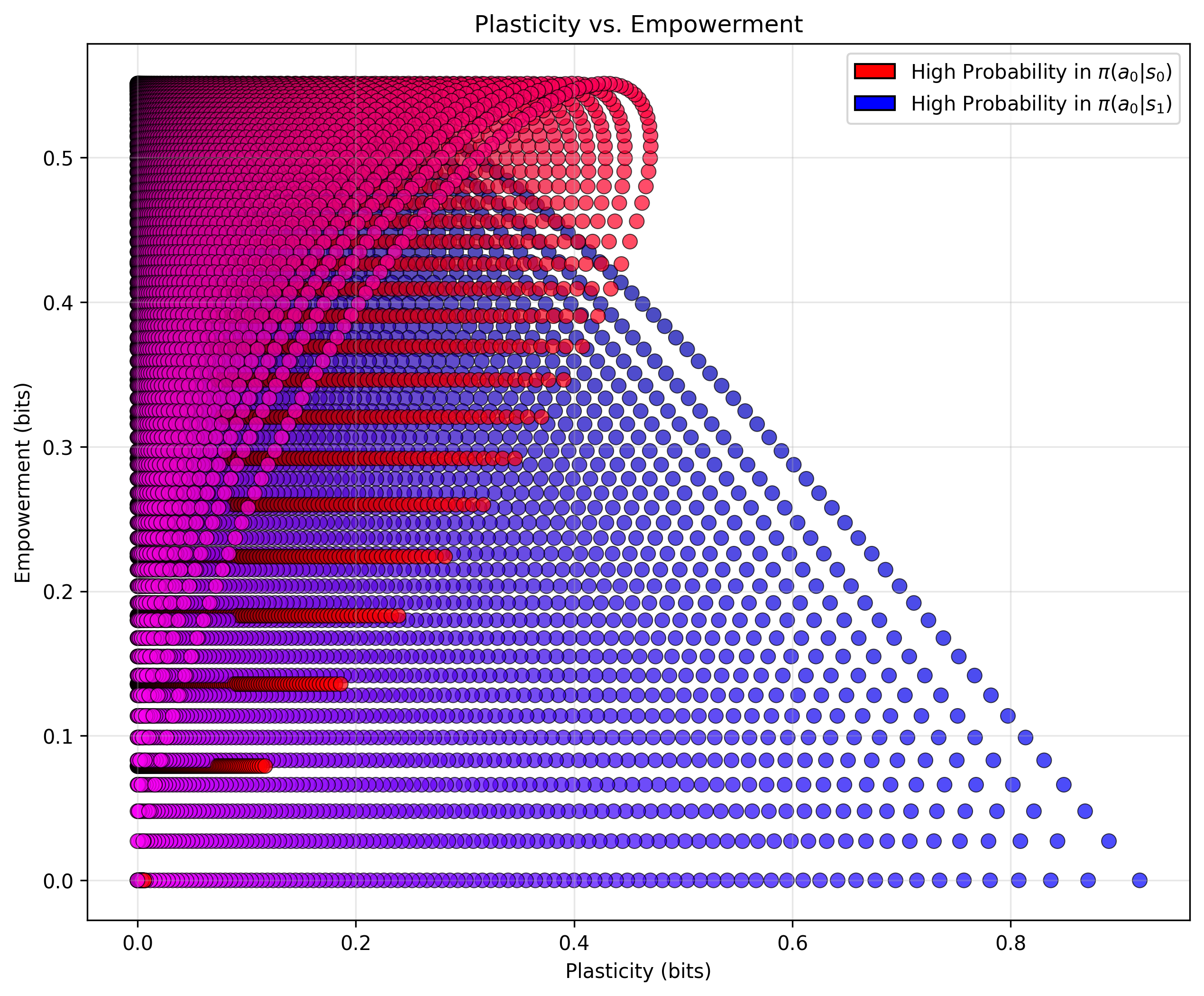}
    \caption{The empirical plasticity-empowerment frontier mapped via policy enumeration.}
    \label{fig:my_image}
\end{figure}

The plasticity-maximizing policy $\pi^*_{\text{plastic}}$ selects actions deterministically, with $\pi^*_{\text{plastic}}(1 \mid 0) = 1$ and $\pi^*_{\text{plastic}}(0 \mid 1) = 1$. Conversely, the empowerment-maximizing policy $\pi^*_{\text{empower}}$ exhibits more distributed behavior: it acts uniformly in state $0$, yielding $\pi^*_{\text{empower}}(a \mid 0) = \frac{1}{2}$ for $a \in \{0, 1\}$, while in state $1$ it selects actions with probabilities $\pi^*_{\text{empower}}(0 \mid 1) = \frac{1}{3}$ and $\pi^*_{\text{empower}}(1 \mid 1) = \frac{2}{3}$. Visually, this divergence is reflected in the policy landscape: highly plastic policies occupy the blue regions of the plot, whereas empowering policies are characterized by a mixture of blue and red regions.

\section{Experiments}
\label{sec:experiments}

\subsection{Control-Gated MDP}
\label{subsec:noise_mdp}
We evaluate our Bellman optimality equation in the Control-Gated MDP. We record each history context value and the current plasticity and empowerment at each iteration.

\begin{figure}[htbp]
    \centering
    \includegraphics[width=0.8\textwidth]{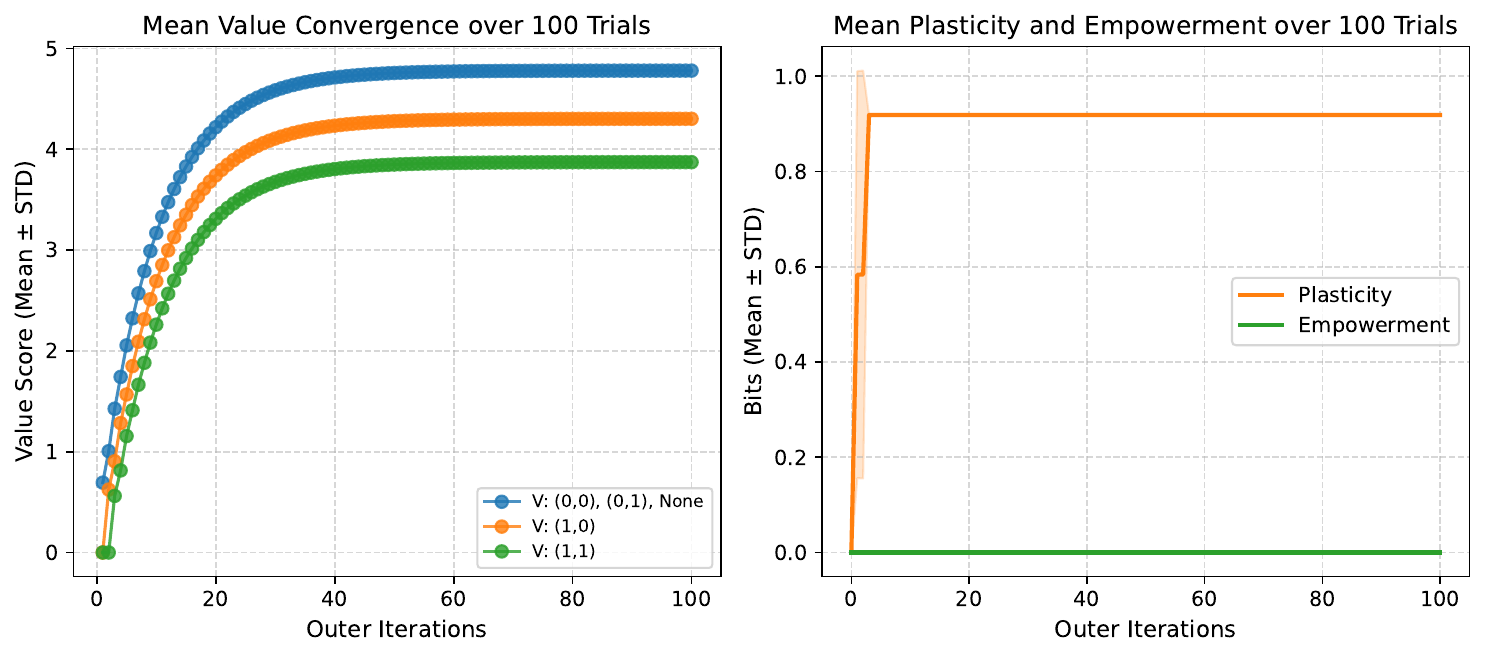}
    \caption{Results from the Control-Gated MDP Benchmark}
    \label{res_control}
\end{figure}

From our results, it is clear that the value iteration is monotonically non-decreasing and converges to the optimal value function across all history contexts. Intuitively, the history contexts with deterministic transitions, $(1,0)$ and $(1,1)$, yield the lowest values. Among these, $(1,1)$ results in the minimum value because it transitions back to state $1$, where the agent again faces low-value prospects, while $(1,0)$ leads to state $0$. 

The empowerment and plasticity results show that the policy ultimately converges to approximately $0.92$ bits of plasticity and $0$ bits of empowerment. The non-zero standard deviation observed in the early iterations is a consequence of the random selection among multiple optimal policies that share the same value but vary in their total plasticity. This highlights the discrepancy between our Bellman optimality equation and the directed information definition of plasticity. 

\subsection{Road Environment}
\label{subsec:road_env}

The Road Environment consists of 5 states $\{-2, -1, 0, 1, 2\}$, signifying the five lanes of a road, and 3 actions $\{-1, 0, 1\}$, describing the directions the agent can turn. At each step, the agent selects an action and the environmental wind is sampled uniformly from $\{-1, 0, 1\}$. The agent's action and the wind are added together to determine the next state of the agent. The transition dynamics are bounded such that neither the agent's action nor the wind can cause the agent to leave the road. The agent starts in a random lane. 
\begin{figure}[htbp]
    \centering
    \includegraphics[width=0.6\textwidth]{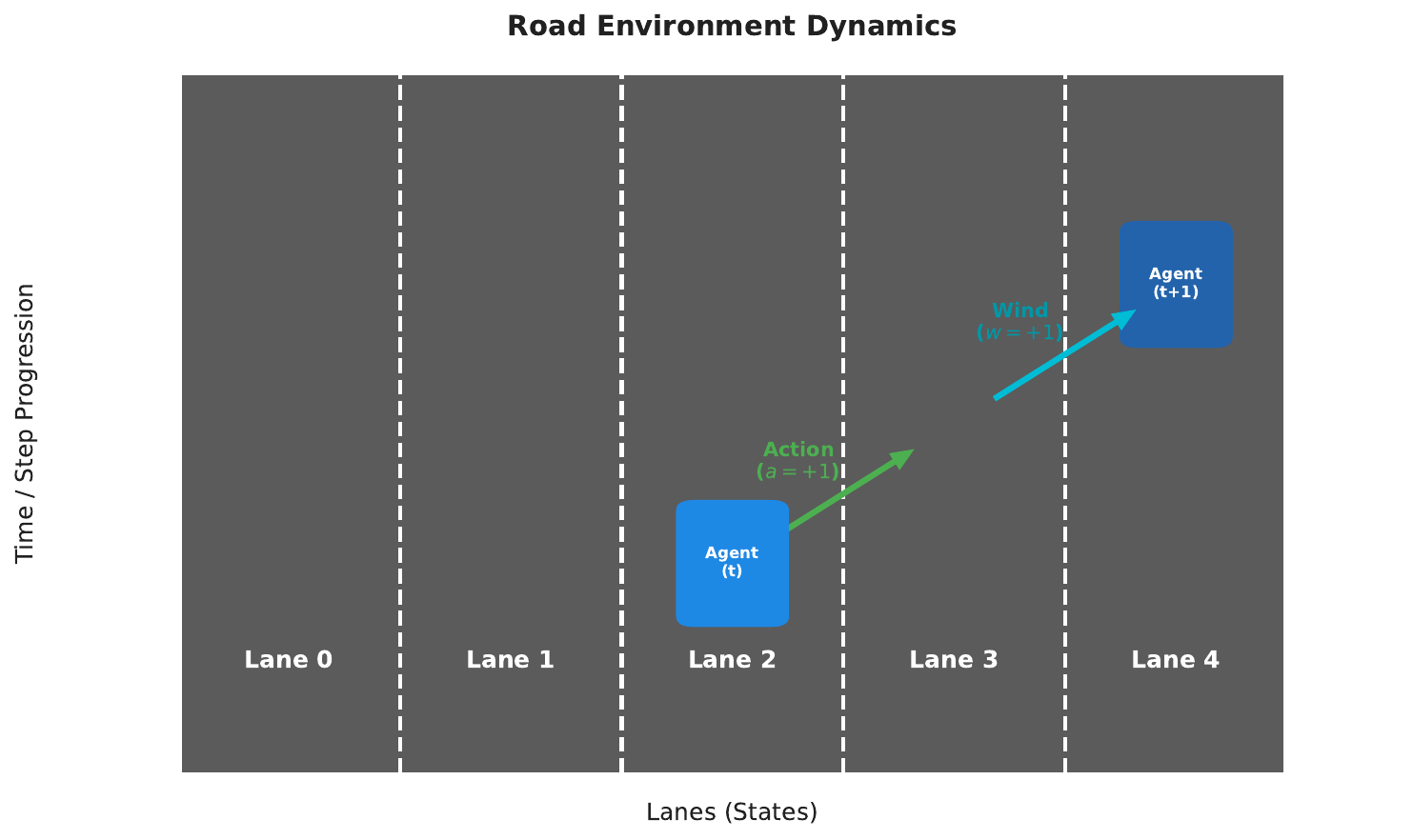}
    \caption{Road Environment}
    \label{fig:road_env}
\end{figure}
\begin{figure}[htbp]
    \centering
    \includegraphics[width=0.8\textwidth]{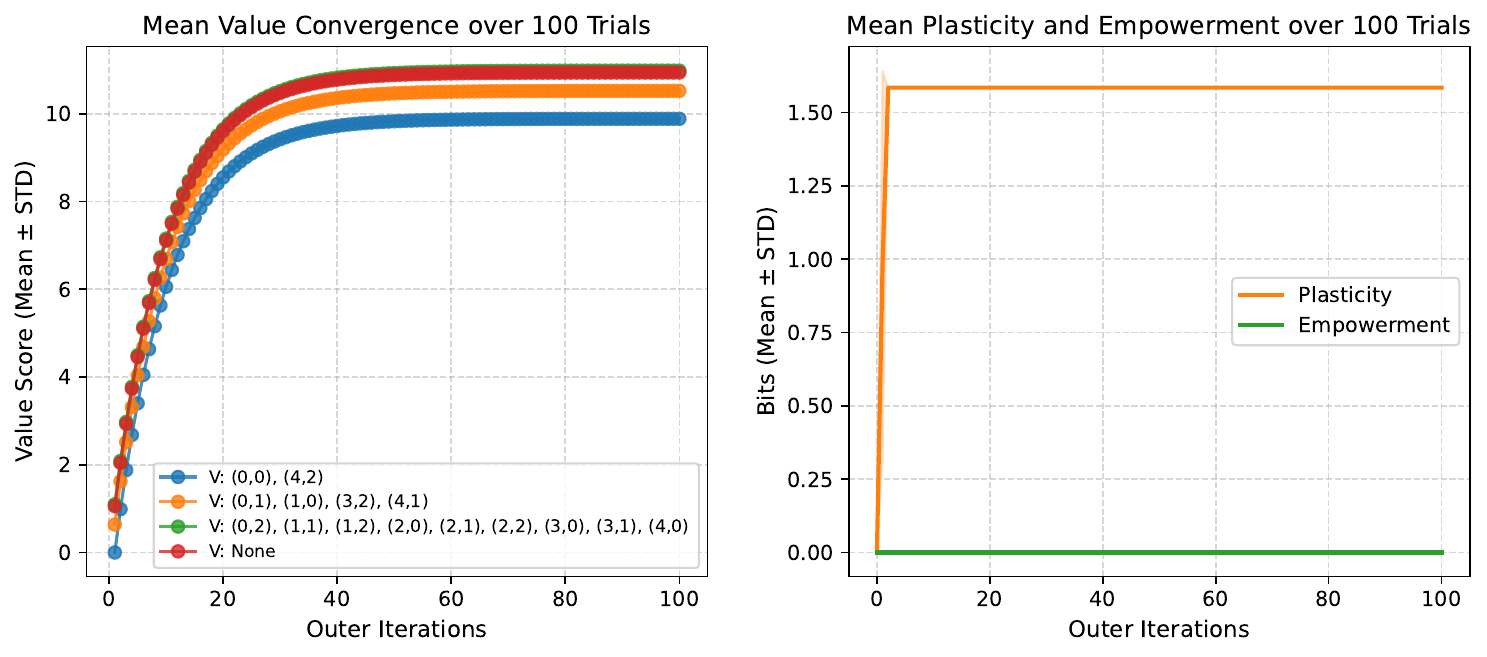}
    \caption{Results from Road Environment}
    \label{fig:res_road_env}
\end{figure}

From these results, attempting to leave the road—corresponding to history contexts $(0,0)$ and $(4,2)$—yields the least amount of value. This is followed by staying straight in the outer lanes, $(0,1)$ and $(4,1)$, along with moving to the outside from the second inner lanes, $(1,0)$ and $(3,2)$. Interestingly, the initial "None" history, where no steps have yet occurred, exhibits a very slightly less plastic value than the inner dynamics of the environment.

\subsection{Two-Room Environment}
This final environment is a $5 \times 7$ gridworld containing $35$ states and a $4$-directional action space, consisting of two distinct rooms. The first room is fully noisy, where all actions result in a transition to a random state within that same room. The second room is fully deterministic, allowing the agent to navigate freely. At the beginning of the episode, the agent must choose which bridge to cross to enter its respective room. If the agent fails to cross a bridge, they fall into a river that continuously flows downward toward the bottom of the grid, where the episode terminates. While the agent is in the river, all state transitions are completely deterministic, and the agent's actions have no effect on their downward flow.

\begin{figure}[htbp]
    \centering
    \includegraphics[width=0.8\textwidth]{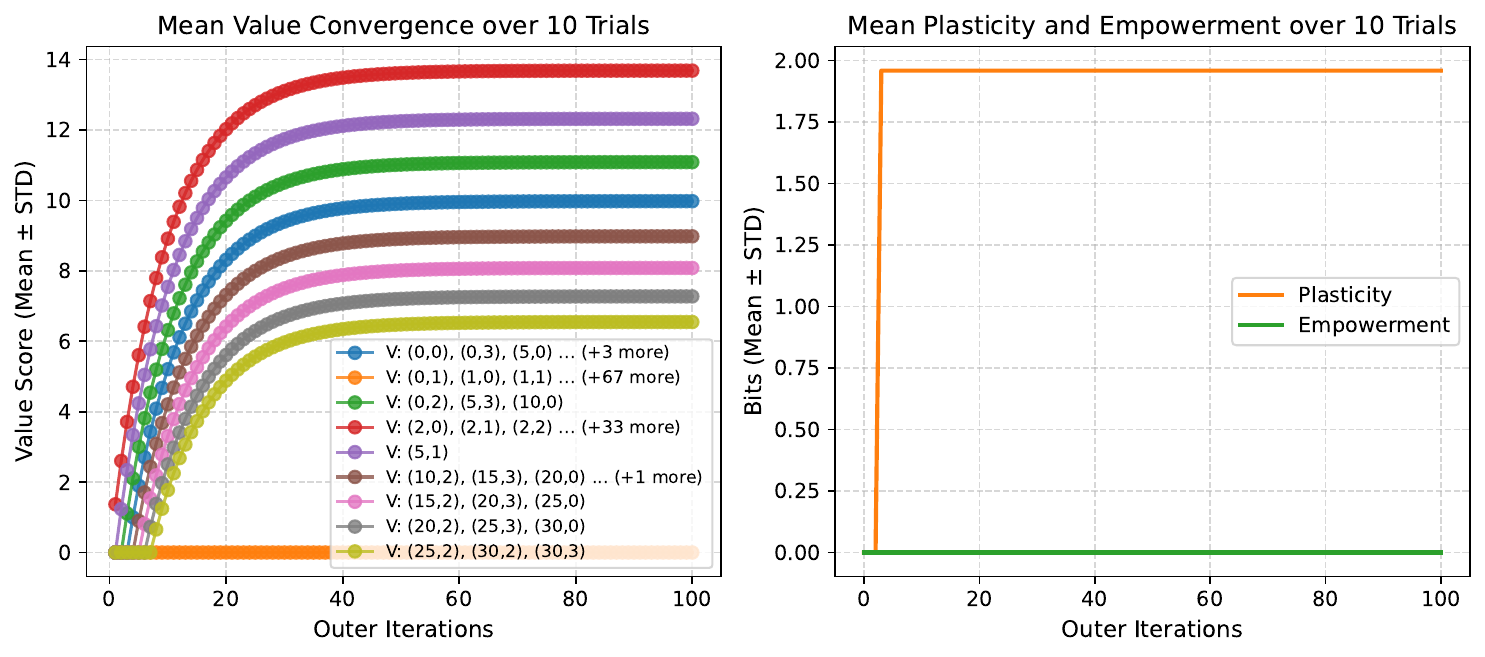}
    \caption{Results from the Two-Room Environment}
    \label{fig:res_two}
\end{figure}

\begin{figure}[htbp]
    \centering
    \includegraphics[width=0.5\textwidth]{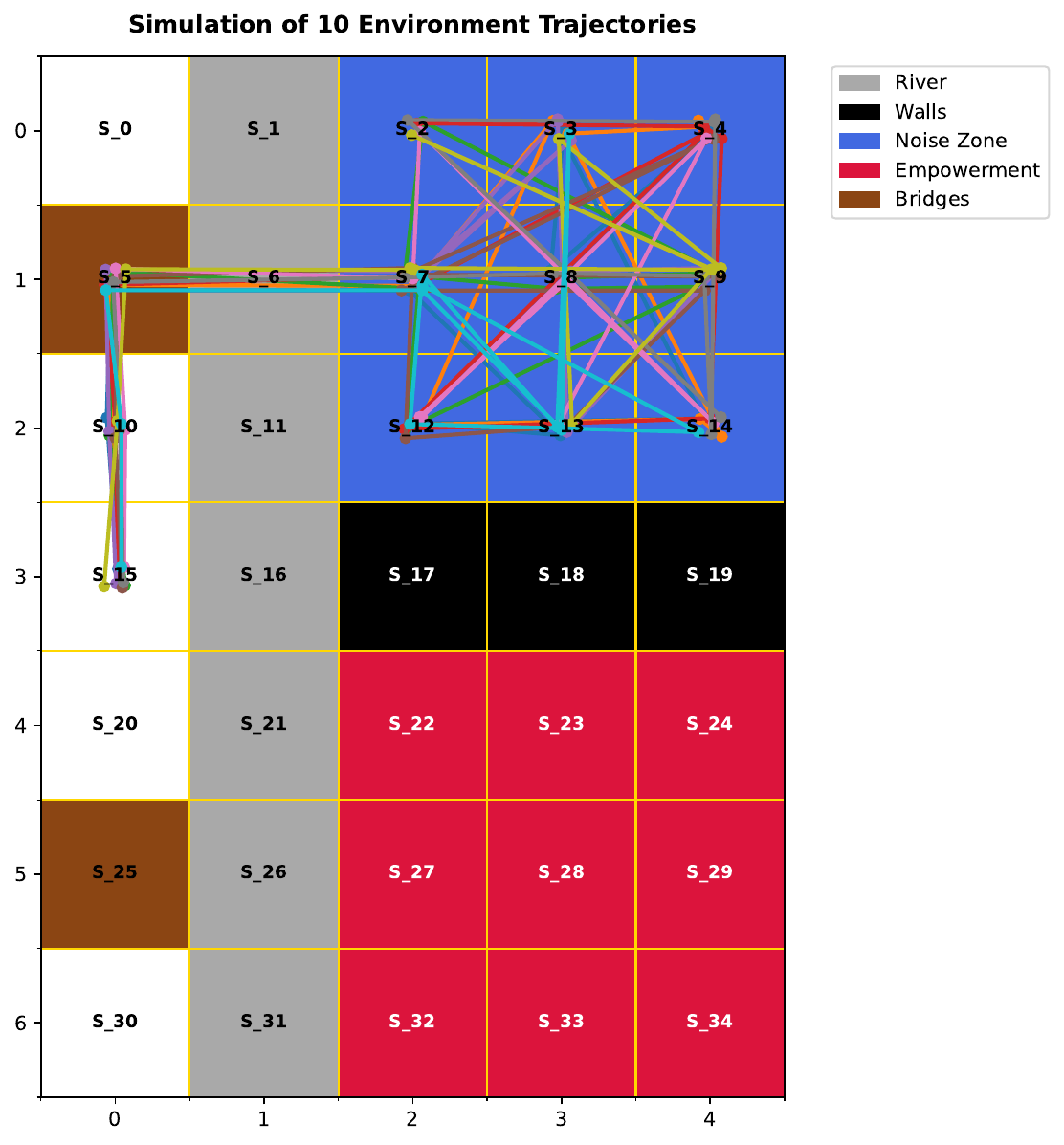}
    \caption{Policy Simulation in the Two-Room Environment}
    \label{fig:pol_two}
\end{figure}
In this gridworld environment, there are too many history contexts to analyze individually, but the policy simulation clearly demonstrates that the agent navigates toward the bridge to the noisy room to capture the maximum plastic value.

\section{Future Work}
\label{sec:future_work}
A natural progression is to develop an online, sample-based algorithm to scale effectively to continuous, high-dimensional environments by parameterizing the policy and value function with deep neural networks. Finally, this framework can be integrated alongside existing empowerment-maximization algorithms which will open the door to agents that can dynamically balance plasticity, empowerment, and external reward during continual learning.

\subsubsection*{Broader Impact Statement}
\label{sec:broaderImpact}

This work introduces a foundational Bellman optimality equation focused on maximizing an agent's plasticity in Markov decision processes. Because it operates at a fundamental algorithmic level, it does not carry immediate negative societal impacts or direct applications to sensitive domains.

\section*{Acknowledgements}
We would like to thank Zihan Wang for inspiration for the Road Environment. 
\appendix

\bibliography{main}
\bibliographystyle{rlj}


\clearpage 
\appendix
\onecolumn

\section{Proofs and Theoretical Derivations}
\label{sec:appendix_proofs}

\subsection{Contraction Mapping Property of the Bellman Operator}

Let $\mathcal{V}$ be the Banach space of bounded, real-valued functions mapping from history state-action space to $\mathbb{R}$. The supremum norm $\| \cdot \|_\infty$ is defined for any $V \in \mathcal{V}$, $\|V\|_\infty = \sup_{s, a} |V(s, a)|$. Let $\gamma \in (0, 1)$ be the discount factor.

We define the Bellman operator $T: \mathcal{V} \to \mathcal{V}$ associated with the value function as:
\begin{equation}
(TV)(s_{t-1}, a_{t-1}) = \max_{\pi \in \Pi(s_{t-1}, a_{t-1})} \underset{s_t \sim p}{\mathbb{E}} \left[ \underset{a_t \sim \pi}{\mathbb{E}} \left[ \log \frac{\pi(a_t \mid s_{t-1}, a_{t-1}, s_t)}{q(a_t \mid s_{t-1}, a_{t-1})} + \gamma V(s_t, a_t) \right] \right]
\end{equation}

\begin{theorem}
The operator $T$ is a $\gamma$-contraction mapping under the supremum norm, i.e., for any $V_1, V_2 \in \mathcal{V}$:
\begin{equation}
\|TV_1 - TV_2\|_\infty \leq \gamma \|V_1 - V_2\|_\infty
\end{equation}
\end{theorem}

\begin{proof}
Let $V_1, V_2 \in \mathcal{V}$ be two arbitrary value functions. For any state-action context $(s_{t-1}, a_{t-1})$, we evaluate the absolute difference $|(TV_1)(s_{t-1}, a_{t-1}) - (TV_2)(s_{t-1}, a_{t-1})|$. 

Using the algebraic property of the maximum, $|\max_x f(x) - \max_x g(x)| \leq \max_x |f(x) - g(x)|$, we can upper-bound the difference as follows:

\begin{equation}
\begin{split}
    \big| (TV_1)(s_{t-1}, a_{t-1}) & - (TV_2)(s_{t-1}, a_{t-1}) \big| \\
    &= \Bigg| \max_{\pi \in \Pi(s_{t-1}, a_{t-1})} \underset{s_t \sim p}{\mathbb{E}} \left[ \underset{a_t \sim \pi}{\mathbb{E}} \left[ \log \frac{\pi(a_t \mid s_{t-1}, a_{t-1}, s_t)}{q(a_t \mid s_{t-1}, a_{t-1})} + \gamma V_1(s_t, a_t) \right] \right] \\
    &\quad - \max_{\pi \in \Pi(s_{t-1}, a_{t-1})} \underset{s_t \sim p}{\mathbb{E}} \left[ \underset{a_t \sim \pi}{\mathbb{E}} \left[ \log \frac{\pi(a_t \mid s_{t-1}, a_{t-1}, s_t)}{q(a_t \mid s_{t-1}, a_{t-1})} + \gamma V_2(s_t, a_t) \right] \right] \Bigg| \\
    &\leq \max_{\pi \in \Pi(s_{t-1}, a_{t-1})} \Bigg| \underset{s_t \sim p}{\mathbb{E}} \left[ \underset{a_t \sim \pi}{\mathbb{E}} \left[ \log \frac{\pi(a_t \mid s_{t-1}, a_{t-1}, s_t)}{q(a_t \mid s_{t-1}, a_{t-1})} + \gamma V_1(s_t, a_t) \right] \right] \\
    &\quad - \underset{s_t \sim p}{\mathbb{E}} \left[ \underset{a_t \sim \pi}{\mathbb{E}} \left[ \log \frac{\pi(a_t \mid s_{t-1}, a_{t-1}, s_t)}{q(a_t \mid s_{t-1}, a_{t-1})} + \gamma V_2(s_t, a_t) \right] \right] \Bigg|
\end{split}
\end{equation}

The identical log terms cancel out within the absolute value:
\begin{align}
|(TV_1)(s_{t-1}, a_{t-1}) - (TV_2)(s_{t-1}, a_{t-1})| &\leq \max_{\pi \in \Pi(s_{t-1}, a_{t-1})} \Bigg| \underset{s_t \sim p}{\mathbb{E}} \left[ \underset{a_t \sim \pi}{\mathbb{E}} \left[ \gamma \big(V_1(s_t, a_t) - V_2(s_t, a_t)\big) \right] \right] \Bigg|
\end{align}

By definition of the supremum norm, $|V_1(s_t, a_t) - V_2(s_t, a_t)| \leq \|V_1 - V_2\|_\infty$ holds for all $(s_t, a_t)$. 
\begin{align}
|(TV_1)(s_{t-1}, a_{t-1}) - (TV_2)(s_{t-1}, a_{t-1})| &\leq \gamma \max_{\pi \in \Pi(s_{t-1}, a_{t-1})} \underset{s_t \sim p}{\mathbb{E}} \left[ \underset{a_t \sim \pi}{\mathbb{E}} \left[ \|V_1 - V_2\|_\infty \right] \right] \nonumber \\
\end{align}

Because valid probability distributions integrate to 1, we can pull the norm out of the expectations:
\begin{align}
|(TV_1)(s_{t-1}, a_{t-1}) - (TV_2)(s_{t-1}, a_{t-1})| &\leq \gamma \|V_1 - V_2\|_\infty
\end{align}

Since this upper bound is uniform and independent of the choice of $(s_{t-1}, a_{t-1})$, taking the supremum over all state-action contexts on the left-hand side completes the proof:
\begin{equation}
\|TV_1 - TV_2\|_\infty = \sup_{s_{t-1}, a_{t-1}} |(TV_1)(s_{t-1}, a_{t-1}) - (TV_2)(s_{t-1}, a_{t-1})| \leq \gamma \|V_1 - V_2\|_\infty
\end{equation}

Given $\gamma \in (0, 1)$, the operator $T$ is a $\gamma$-contraction mapping. By the Banach Fixed-Point Theorem, $T$ possesses a unique fixed point $V^*$.
\end{proof}

\subsection{Limiting the Bellman Operator to Policy Endpoints}

We now show that the optimization domain of the Bellman operator $T$ can be restricted entirely to deterministic policies without altering the operator itself.

\begin{theorem}
For any value function $V \in \mathcal{V}$, the objective function inside the Bellman operator optimization problem:

\begin{equation}
f(\pi) = \underset{s_t \sim p}{\mathbb{E}} \left[ \underset{a_t \sim \pi}{\mathbb{E}} \left[ \log \frac{\pi(a_t \mid s_{t-1}, a_{t-1}, s_t)}{q(a_t \mid s_{t-1}, a_{t-1})} + \gamma V(s_t, a_t) \right] \right]
\end{equation}

is convex with respect to $\pi$. Consequently, the maximum of $f(\pi)$ over the convex compact set $\Pi$ is attained at an extreme point of the policy space, corresponding to a deterministic action selection.
\end{theorem}

\begin{proof}
Let $V \in \mathcal{V}$ be an arbitrary value function. To analyze the convexity of $f(\pi)$ with respect to $\pi$, we can expand the inner expectation for a fixed next-state $s_t$. For structural clarity, let $p_i = \pi(a_i \mid s_{t-1}, a_{t-1}, s_t)$, $q_i = q(a_i \mid s_{t-1}, a_{t-1})$, and $V_i = \gamma V(s_t, a_i)$ where $i$ indexes the finite action space $\mathcal{A}$. 

We can rewrite the core objective function as a sum of three distinct components mapping from the probability simplex $\Delta(\mathcal{A}) \to \mathbb{R}$:
\begin{equation}
\phi(\mathbf{p}) = \sum_{i=1}^{|\mathcal{A}|} p_i \log p_i - \sum_{i=1}^{|\mathcal{A}|} p_i \log q_i + \sum_{i=1}^{|\mathcal{A}|} p_i V_i
\end{equation}

We will establish the convexity of $\phi(\mathbf{p})$ by analyzing each term independently:

\paragraph{1. Convexity of the negative entropy term:}
Let $g(p_i) = p_i \log p_i$ defined on the domain $p_i \in (0, 1]$. We prove convexity by showing that the second derivative of $p_i \log p_i$ is non-negative for $p_i > 0$. Evaluating the first and second derivatives with respect to $p_i$ yields:
\begin{equation}
g'(p_i) = \log p_i + 1, \quad \text{and} \quad g''(p_i) = \frac{1}{p_i}
\end{equation}
Since $p_i > 0$ for any valid probability vector inside the simplex, the second derivative is strictly positive ($g''(p_i) > 0$). This guarantees that $g(p_i)$ is strictly convex. Because the sum of convex functions preserves convexity, the aggregate term $\sum_{i=1}^{|\mathcal{A}|} p_i \log p_i$ is strictly convex.

\paragraph{2. Convexity of the cross-entropy baseline term:}
Let $h(p_i) = -p_i \log q_i$. Because the baseline distribution $q$ is fixed and independent of the optimization policy $\pi$, the value $-\log q_i$ treats itself as a constant coefficient. Taking the derivatives with respect to $p_i$ results in:
\begin{equation}
h'(p_i) = -\log q_i, \quad \text{and} \quad h''(p_i) = 0
\end{equation}
Since the second derivative is zero everywhere, $h(p_i)$ is an affine (linear) function. All affine functions are convex, implying the sum $-\sum_{i=1}^{|\mathcal{A}|} p_i \log q_i$ is a convex function of $\mathbf{p}$.

\paragraph{3. Convexity of the value function expectation term:}
Similarly, the value function parameters $V_i$ are independent of $\mathbf{p}$. The sum $\sum_{i=1}^{|\mathcal{A}|} p_i V_i$ forms a linear combination of the coordinates of $\mathbf{p}$, which makes it an affine function, and therefore it is inherently convex.

\bigskip
In conclusion, because $\phi(\mathbf{p})$ is a sum consisting of a strictly convex function and two affine functions, the entire objective function $\phi(\mathbf{p})$ is strictly convex across the action distribution vector space. Since the outer expectation $\mathbb{E}_{s_t \sim p}$ acts as a linear operator it preserves convexity. Therefore, the global objective function $f(\pi)$ is convex with respect to $\pi$.

By Bauer's Maximum Principle, a convex function optimized over a compact, convex set (the probability simplex $\Delta(\mathcal{A})$) must attain its maximum at one of its extreme points. The extreme points of a probability simplex are the vertices where the entire probability mass is allocated to a single choice:

\begin{equation}
\pi^*(a \mid s_{t-1}, a_{t-1}, s_t) = \delta(a - a^*)
\end{equation}
where $\delta$ is the Kronecker delta and $a^* \in \mathcal{A}$ represents a deterministic action choice. Consequently, limiting the policy search space exclusively to these deterministic endpoints yields an identical value under the operator.
\end{proof}
\clearpage
\section{Experimental Details}
\label{sec:supp_hyperparameters}

\begin{table}[h!]
\centering
\caption{Hyperparameter configurations for the Control-Gated MDP, Road Environment and Two-Room Environment.}
\label{tab:hyperparameters}
\vskip 0.15in
\begin{tabular}{lccc}
\hline
\textbf{Hyperparameter} & \textbf{Control-Gated MDP} & \textbf{Road Environment} & \textbf{Two-Room Environment} \\ \hline
Number of Trials ($T$) & 100 & 100 & 10 \\
Total Iterations ($N$) & 100 & 100 & 100\\
Discount Factor ($\gamma$) & 0.9 & 0.9 & 0.9 \\
\hline
\end{tabular}
\end{table}

\section{Code and Reproducibility}
\label{sec:code_reproducibility}

The complete source code to reproduce all experiments in this paper is publicly available on GitHub at:
\begin{center}
    \url{https://github.com/jeremy-lucas-mcgill/A-Bellman-Optimality-Equation-for-Plasticity}
\end{center}

\end{document}